\documentclass{article}
\usepackage[final]{nips_2016}

\usepackage[utf8]{inputenc} 
\usepackage[T1]{fontenc}    
\usepackage{url}            
\usepackage{booktabs}       
\usepackage{amsfonts}       
\usepackage{nicefrac}       
\usepackage{microtype}      
\usepackage{bbm}

\usepackage{subfigure}
\usepackage{times}
\usepackage{url}
\usepackage{color}
\usepackage{amsmath}
\usepackage{amssymb}
\usepackage{multirow}
\usepackage{multicol}
\usepackage{comment}
\usepackage{multicol}
\usepackage{algorithm}
\usepackage{algorithmic}
\usepackage{xspace}
\usepackage{graphicx}
\usepackage{amsthm}

\newif\ifspacetight
\spacetighttrue

\ifspacetight

\definecolor{grey}{rgb}{0.7,0.7,0.7}

\makeatletter
\let\@listiold\@listi
\let\@listiiold\@listii
\let\@listiiiold\@listiii
\let\@listivold\@listiv
\let\@listvold\@listv

\makeatother
\fi

\newcommand{\argmin}{\operatornamewithlimits{argmin}}

\newcommand{\ie}{{\it i.e.}~}
\newcommand{\eg}{{\it e.g.}~}

\newtheorem{definition}{Definition}

\newtheorem{proposition}{Proposition}

\usepackage[algo2e, noend, noline, linesnumbered, ruled]{algorithm2e}

\providecommand{\DontPrintSemicolon}{\dontprintsemicolon}
\DontPrintSemicolon
\makeatletter
\newcommand{\pushline}{\Indp}
\newcommand{\popline}{\Indm}
\makeatother

\newcommand{\beq}{\begin{equation}}
\newcommand{\eeq}{\end{equation}}

\newcommand{\beqa}{\begin{eqnarray}}
\newcommand{\eeqa}{\end{eqnarray}}

\newcommand{\beqan}{\begin{eqnarray*}}
\newcommand{\eeqan}{\end{eqnarray*}}

\def\ie{{\em i.e.,}\xspace}

\title{Memory-Efficient Backpropagation Through Time}

\author{
 Audr{\=u}nas Gruslys \\ Google DeepMind \\ \texttt{audrunas@google.com}  \And
 Remi Munos \\ Google DeepMind \\ \texttt{munos@google.com} \And
 Ivo Danihelka \\ Google DeepMind \\ \texttt{danihelka@google.com} \And
 Marc Lanctot \\ Google DeepMind \\ \texttt{lanctot@google.com} \And
 Alex Graves \\ Google DeepMind \\ \texttt{gravesa@google.com}
}

\begin{document}

\maketitle

\begin{abstract}
We propose a novel approach to reduce memory consumption of the backpropagation through time (BPTT) algorithm when training recurrent neural networks (RNNs). Our approach uses dynamic programming to balance a trade-off between caching of intermediate results and recomputation. The algorithm is capable of tightly fitting within almost any user-set memory budget while finding an optimal execution policy minimizing the computational cost. Computational devices have limited memory capacity and maximizing a computational performance given a fixed memory budget is a practical use-case. We provide asymptotic computational upper bounds for various regimes. The algorithm is particularly effective for long sequences. For sequences of length 1000, our algorithm saves 95\% of memory usage while using only one third more time per iteration than the standard BPTT.
\end{abstract}


\section{Introduction}

Recurrent neural networks (RNNs) are artificial neural networks where connections between units can form cycles. They are often used for sequence mapping problems, as they can propagate hidden state information from early parts of the sequence back to later points. LSTM (\citet{hochreiter1997long}) in particular is an RNN architecture that has excelled in sequence generation (\citet{eck2002first,sutskever2011generating,Graves2012}), speech recognition (\citet{graves2013speech}) and reinforcement learning (\citet{sorokin2015deep, mnih2016asynchronous}) settings. Other successful RNN architectures include Neural Turing Machines (NTM) (\citet{graves2014neural}), DRAW network (\citet{gregor2015draw}), Neural Transducers (\citet{grefenstette2015learning}).

Backpropagation Through Time algorithm (BPTT) (\citet{rumelhart1985learning,werbos1990backpropagation}) is typically used to obtain gradients during training. One important problem is the large memory consumption required by the BPTT. This is especially troublesome when using Graphics Processing Units (GPUs) due to the limitations of GPU memory.

Memory budget is typically known in advance. Our algorithm balances the tradeoff between memorization and recomputation by finding an optimal memory usage policy which minimizes the total computational cost for any fixed memory budget. The algorithm exploits the fact that the same memory slots may be reused multiple times. The idea to use dynamic programming to find a provably optimal policy is the main contribution of this paper. 

Our approach is largely architecture agnostic and works with most recurrent neural networks. Being able to fit within limited memory devices such as GPUs will typically compensate for any increase in computational cost.

\section{Background and related work}

In this section, we describe the key terms and relevant previous work for memory-saving in RNNs.

\begin{definition}
An {\bf RNN core} is a feed-forward neural network which is cloned (unfolded in time) repeatedly, where each clone represents a particular time point in the recurrence.
\end{definition}

For example, if an RNN has a single hidden layer whose outputs feed back into the same hidden layer, then for a sequence length of $t$ the unfolded network is feed-forward and contains $t$ RNN cores.

\begin{definition}
The {\bf hidden state} of the recurrent network is the part of the output of the RNN core which is passed into the next RNN core as an input.
\end{definition}

In addition to the initial hidden state, there exists a single hidden state per time step once the network is unfolded.

\begin{definition}
\label{def-internal}
The {\bf internal state} of the RNN core for a given time-point is all the necessary information required to backpropagate gradients over that time step once an input vector, a gradient with respect to the output vector, and a gradient with respect to the output hidden state is supplied. We define it to also include an output hidden state.
\end{definition}

An internal state can be (re)evaluated by executing a single forward operation taking the previous hidden state and the respective entry of an input sequence as an input. For most network architectures, the internal state of the RNN core will include a hidden input state, as this is normally required to evaluate gradients. This particular choice of the definition will be useful later in the paper.

\begin{definition}
A {\bf memory slot} is a unit of memory which is capable of storing a single hidden state or a single internal state (depending on the context).
\end{definition}

\subsection{Backpropagation through Time}

Backpropagation through Time (BPTT) (\citet{rumelhart1985learning,werbos1990backpropagation}) is one of the commonly used techniques to train recurrent networks. BPTT ``unfolds'' the neural network in time by creating several copies of the recurrent units which can then be treated like a (deep) feed-forward network with tied weights. Once this is done, a standard forward-propagation technique can be used to evaluate network fitness over the whole sequence of inputs, while a standard backpropagation algorithm can be used to evaluate partial derivatives of the loss criteria with respect to all network parameters. This approach, while being computationally efficient is also fairly intensive in memory usage. This is because the standard version of the algorithm effectively requires storing internal states of the unfolded network core at every time-step in order to be able to evaluate correct partial derivatives.

\subsection{Trading memory for computation time}

The general idea of trading computation time and memory consumption in general computation graphs
has been investigated in the automatic differentiation community (\citet{dauvergne2006data}). Recently, the rise of deep architectures and recurrent networks has increased interest in a less general case where the graph of forward computation is a chain and gradients have to be chained in a reverse order. This simplification leads to relatively simple memory-saving strategies and heuristics. In the context of BPTT, instead of storing hidden network states, some of the intermediate results can be recomputed on demand by executing an extra forward operation.

Chen et. al. proposed subdividing the sequence of size $t$ into $\sqrt{t}$ equal parts and memorizing only hidden states between the subsequences and all internal states within each segment (\citet{chen2016}). This uses $O(\sqrt{t})$ memory at the cost of making an additional forward pass on average, as once the errors are backpropagated through the right-side of the sequence, the second-last subsequence has to be restored by repeating a number of forward operations. We refer to this as Chen's $\sqrt{t}$ algorithm.

The authors also suggest applying the same technique recursively several times by sub-dividing the sequence into $k$ equal parts and terminating the recursion once the subsequence length becomes less than $k$. The authors have established that this would lead to memory consumption of $O(k \log_{k+1}(t))$ and computational complexity of $O(t \log_{k}(t))$. 
This algorithm has a minimum possible memory usage of $\log_2(t)$ in the case when $k=1$. We refer to this as Chen's {\it recursive} algorithm.

\section{Memory-efficient backpropagation through time}

We first discuss two simple examples: when memory is very scarce, and when it is somewhat limited.

When memory is very scarce, it is straightforward to design a simple but computationally inefficient algorithm for backpropagation of errors on RNNs which only uses only a constant amount of memory. Every time when the state of the network at time $t$ has to be restored, the algorithm would simply re-evaluate the state by forward-propagating inputs starting from the beginning until time $t$. As backpropagation happens in the reverse temporal order, results from the previous forward steps can not be reused (as there is no memory to store them). This would require repeating $t$ forward steps before backpropagating gradients one step backwards (we only remember inputs and the initial state). This would produce an algorithm requiring $t(t+1)/2$ forward passes to backpropagate errors over $t$ time steps. The algorithm would be $O(1)$ in space and $O(t^2)$ in time.


A simple way of reducing memory consumption is simply to store only hidden RNN states at all time points. When errors have to be backpropagated from time $t$ to $t-1$, an internal RNN core state can be re-evaluated by executing another forward operation taking the previous hidden state as an input. The backward operation can follow immediately. This approach can lead to fairly significant memory savings, as typically the recurrent network hidden state is much smaller than an internal state of the network core itself. On the other hand this leads to another forward operation being executed during the backpropagation stage.

\subsection{Backpropagation though time with selective hidden state memorization (BPTT-HSM)}
\label{section-hidden}

The idea behind the proposed algorithm is to compromise between two previous extremes. Suppose that we want to forward and backpropagate a sequence of length $t$, but we are only able to store $m$ hidden states in memory at any given time. We may reuse the same memory slots to store different hidden states during backpropagation. Also, suppose that we have a single RNN core available for the purposes of intermediate calculations which is able to store a single internal state. Define $C(t, m)$ as a computational cost of backpropagation measured in terms of how many forward-operations one has to make in total during forward and backpropagation steps combined when following an optimal memory usage policy minimizing the computational cost.
One can easily set the boundary conditions: $C(t, 1) = \frac{1}{2} t(t+1) $ is the cost of the minimal memory approach, while $C(t, m) = 2t - 1$ for all $m \geq t$ when memory is plentiful (as shown in Fig. \ref{illustration-hidden} a). 
Our approach is illustrated in Figure~\ref{drawing-hidden}. Once we start forward-propagating steps at time $t=t_0$, at any given point $y > t_0$ we can choose to put the current hidden state into memory (step 1). This step has the cost of $y$ forward operations. States will be read in the reverse order in which they were written: this allows the algorithm to store states in a stack. Once the state is put into memory at time $y=D(t, m)$, we can reduce the problem into two parts by using a divide-and-conquer approach: running the same algorithm on the $t > y$ side of the sequence while using $m-1$ of the remaining memory slots at the cost of $C(t - y, m - 1)$ (step 2), and then reusing $m$ memory slots when backpropagating on the $t \leq y$ side at the cost of $C(y, m)$ (step 3). We use a full size $m$ memory capacity when performing step 3 because we could release the hidden state $y$ immediately after finishing step 2.

\begin{figure}[!hb]
  \centering
    \includegraphics[width=1.00\textwidth]{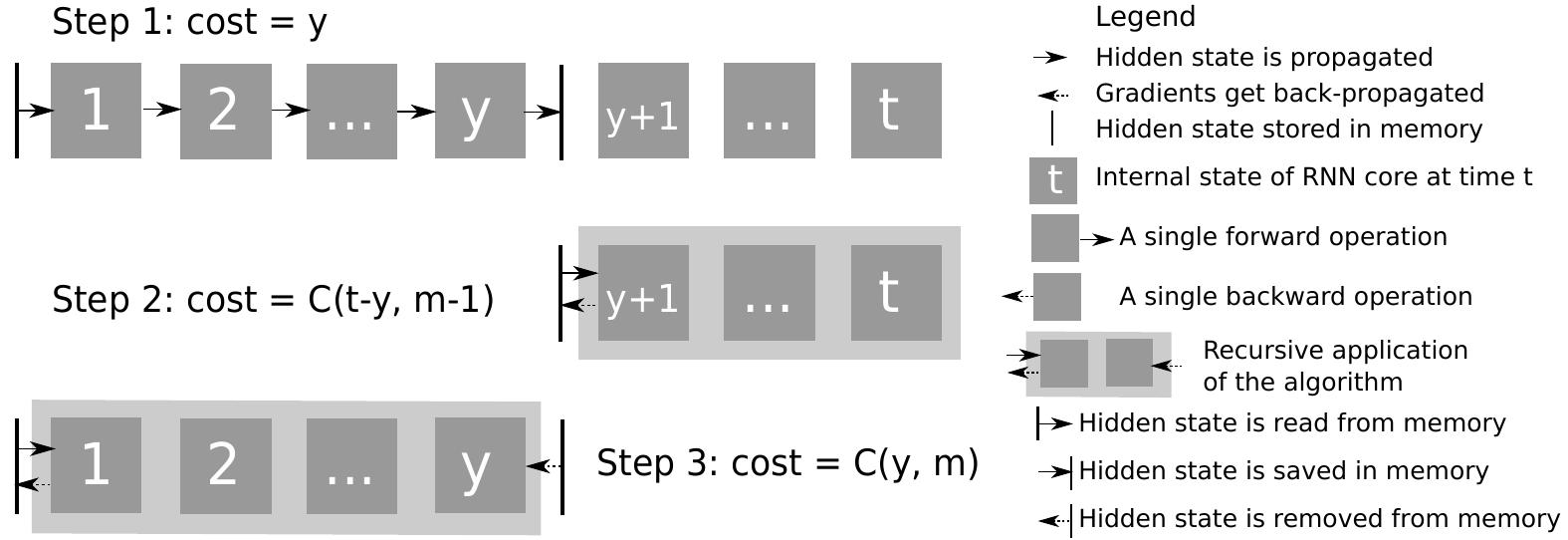}
  \caption{The proposed divide-and-conquer approach.}
  \label{drawing-hidden}
\end{figure}

\begin{figure}[!ht]
\centering
\subfigure[Theoretical computational cost measured in number of forward operations per time step.]{\label{cost-hidden}
\includegraphics[width=0.4\textwidth]{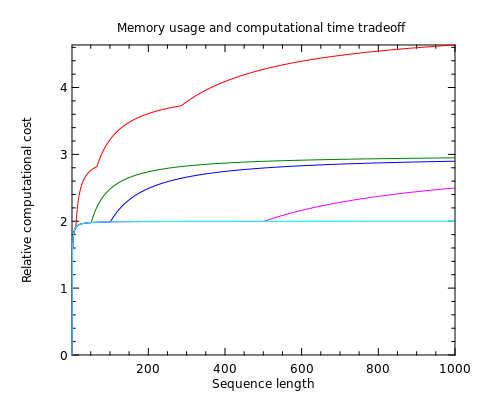}}
\subfigure[Benchmarked computational cost measured in miliseconds.]{  \label{real-times}
\includegraphics[width=0.4\textwidth]{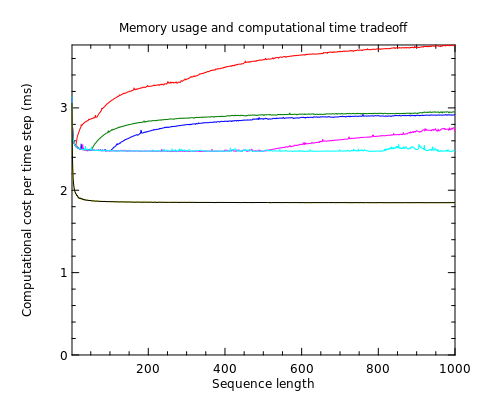}}
  \caption{Computational cost per time-step when the algorithm is allowed to remember 10 (red), 50 (green), 100 (blue), 500 (violet), 1000 (cyan) hidden states. The grey line shows the performance of standard BPTT without memory constraints;
(b) also includes a large constant value caused by a single backwards step per time step which was excluded from the theoretical computation, which value makes a relative performance loss much less severe in practice than in theory.}
\end{figure}

The base case for the recurrent algorithm is simply a sequence of length $t=1$ when forward and backward propagation may be done trivially on a single available RNN network core. This step has the cost $C(1, m) = 1$.

Having established the protocol we may find an optimal policy $D(t, m)$. Define the cost of choosing the first state to be pushed at position $y$ and later following the optimal policy as: 

\begin{equation}
\label{eq-q}
Q(t, m, y) = y + C(t - y, m - 1) + C(y, m)
\end{equation}
\begin{multicols}{2}
\begin{equation}
\label{eq-c}
C(t, m) = Q(t, m, D(t, m))
\end{equation}
~~~~
\begin{equation}
\label{eq-d}
D(t, m) =  \argmin_{1 \le y < t} Q(t, m, y)
\end{equation}
\end{multicols}

\begin{figure}[!hb]
  \centering
    \includegraphics[width=1.0\textwidth]{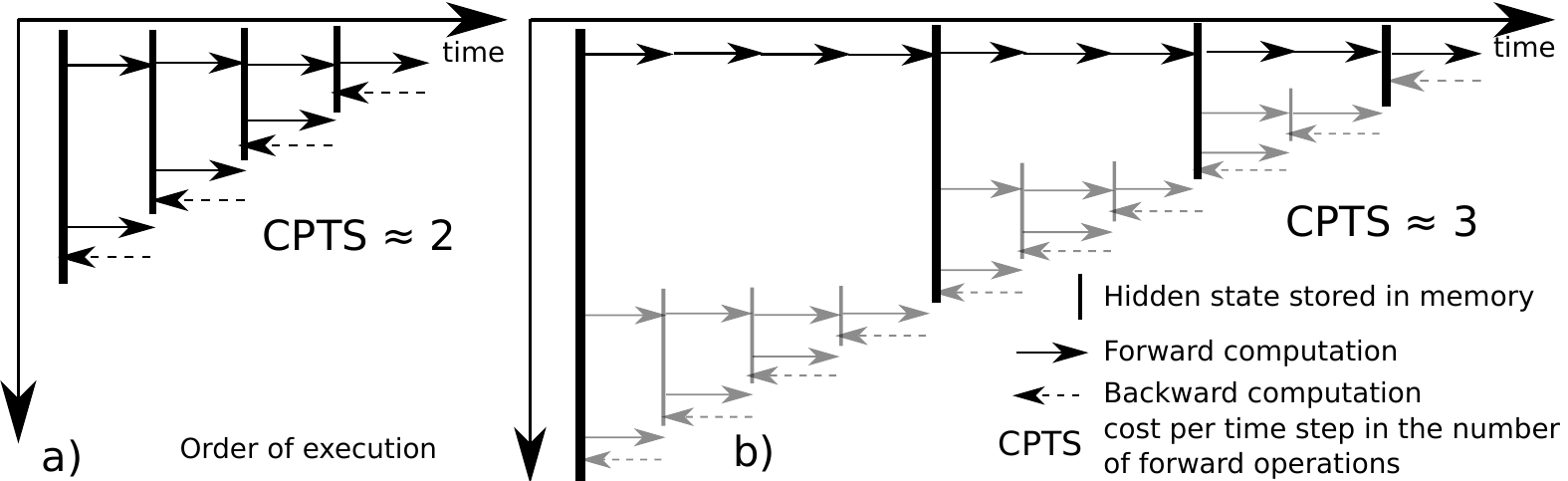}
  \caption{Illustration of the optimal policy for $m = 4$ and a) $t=4$ and b) $t = 10$. Logical sequence time goes from left to right, while execution happens from top to the bottom.}
  \label{illustration-hidden}
\end{figure}

Equations can be solved exactly by using dynamic programming subject to the boundary conditions established previously (\eg as in Figure \ref{cost-hidden}). $D(t, m)$ will determine the optimal policy to follow. Pseudocode is given in the supplementary material. Figure \ref{illustration-hidden} illustrates an optimal policy found for two simple cases.

\subsection{Backpropagation though time with selective internal state memorization (BPTT-ISM)}
\label{section-full}

Saving internal RNN core states instead of hidden RNN states would allow us to save a single forward operation during backpropagation in every divide-and-conquer step, but at a higher memory cost. Suppose we have a memory capacity capable of saving exactly $m$ internal RNN states. First, we need to modify the boundary conditions: $C(t, 1) = \frac{1}{2} t(t+1) $ is a cost reflecting the minimal memory approach, while $C(t, m) = t$ for all $m \ge t$ when memory is plentiful (equivalent to standard BPTT).

As previously, $C(t, m)$ is defined to be the computational cost for combined forward and backward propagations over a sequence of length $t$ with memory allowance $m$ while following an optimal memory usage policy.
As before, the cost is measured in terms of the amount of total forward steps made, because the number of backwards steps is constant.
Similarly to BPTT-HSM, the process can be divided into parts using divide-and-conquer approach (Fig \ref{drawing-full}). For any values of $t$ and $m$ position of the first memorization $y=D(t, m)$ is evaluated. $y$ forward operations are executed and an internal RNN core state is placed into memory. This step has the cost of $y$ forward operations (Step 1 in Figure \ref{drawing-full}). As the internal state also contains an output hidden state, the same algorithm can be recurrently run on the high-time (right) side of the sequence while having on less memory slot available (Step 2 in Figure \ref{drawing-full}). This step has the cost of $C(t-y, m-1)$ forward operations. Once gradients are backpropagated through the right side of the sequence, backpropagation can be done over the stored RNN core (Step 3 in Figure \ref{drawing-full}). This step has no additional cost as it involves no more forward operations. The memory slot can now be released leaving $m$ memory available. Finally, the same algorithm is run on the left-side of the sequence (Step 4 in Figure \ref{drawing-full}). This final step has the cost of $C(y-1, m)$ forward operations. Summing the costs gives us the following equation:
\begin{equation}
\label{eq-q-internal}
Q(t, m, y) =  y + C(y - 1, m) + C(t - y, m - 1)
\end{equation}
Recursion has a single base case: backpropagation over an empty sequence is a nil operation which has no computational cost making $C(0, m) = 0$.

\begin{figure}[!ht]
  \centering
    \includegraphics[width=1.0\textwidth]{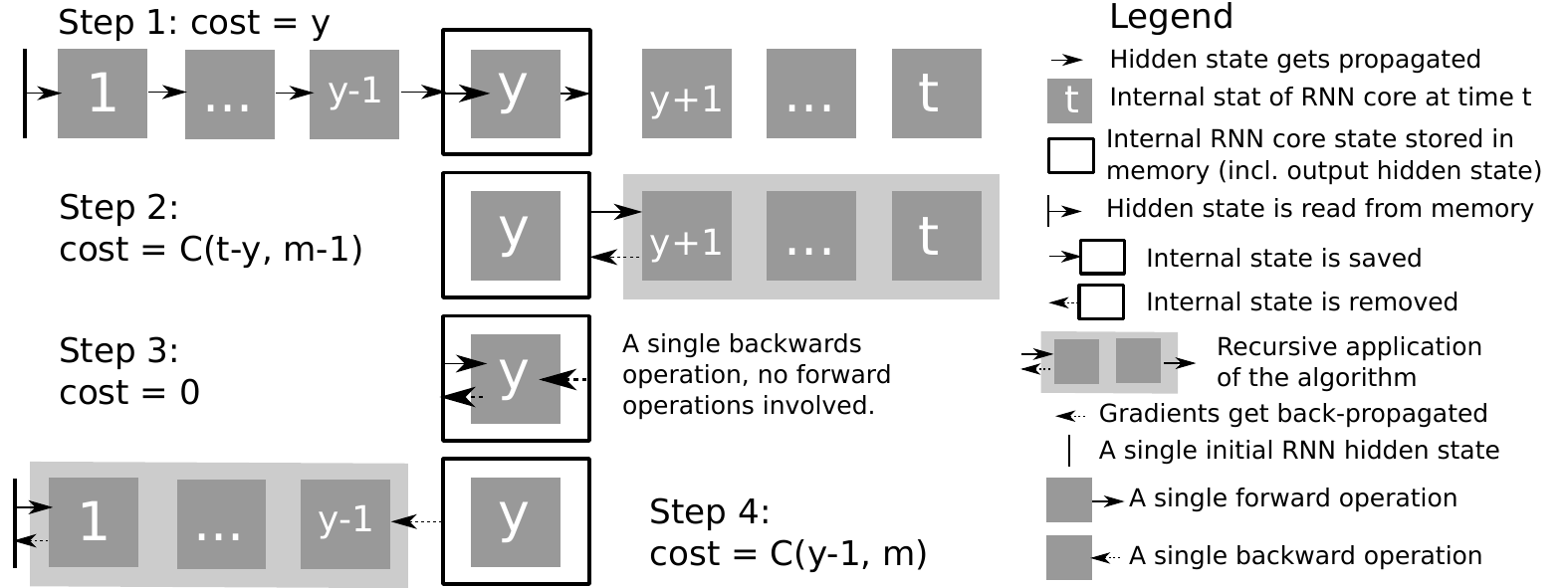}
  \caption{Illustration of the divide-and-conquer approach used by BPTT-ISM.}
  \label{drawing-full}
\end{figure}

Compared to the previous section (\ref{eq-c}) stays the same while (\ref{eq-d}) is minimized over $1 \le y \le t$ instead of $1 \le y < t$. This is because it is meaningful to remember the last internal state while there was no reason to remember the last hidden state. A numerical solution of $C(t, m)$ for several different memory capacities is shown in Figure~\ref{cost-full}.
\begin{equation}
\label{eq-d-internal}
D(t, m) =  \argmin_{1 \le y \le t} Q(t, m, y)
\end{equation}

\begin{figure}[!ht]
\centering
\subfigure[BPTT-ISM (section \ref{section-full}).]{\label{cost-full}
\includegraphics[width=0.4\textwidth]{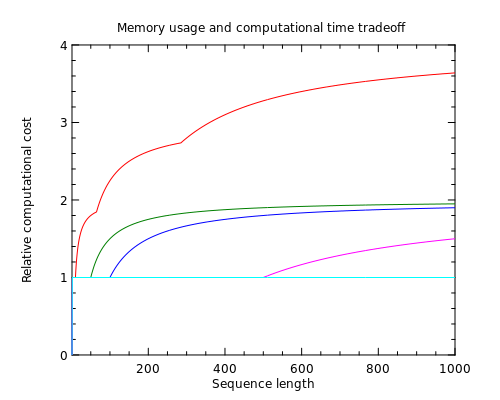}}
~~~~~~~~
\subfigure[BPTT-MSM (section \ref{sec-mixed}).]{  \label{cost-mixed}
\includegraphics[width=0.4\textwidth]{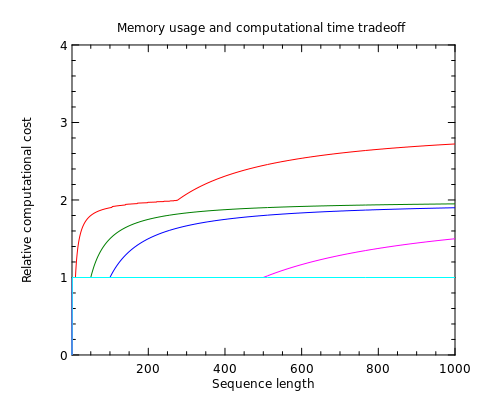}}
  \caption{Comparison of two backpropagation algorithms in terms of theoretical costs. Different lines show the number of forward operations per time-step when the memory capacity is limited to 10 (red), 50 (green), 100 (blue), 500 (violet), 1000 (cyan) internal RNN core states. Please note that the units of memory measurement are different than in Figure \ref{cost-hidden} (size of an internal state vs size of a hidden state). It was assumed that the size of an internal core state is $\alpha=5$ times larger than the size of a hidden state. The value of $\alpha$ influences only the right plot. All costs shown on the right plot should be less than the respective costs shown on the left plot for any value of $\alpha$.}
\end{figure}
As seen in Figure~\ref{cost-full}, our methodology saves 95\% of memory for sequences of 1000 (excluding input vectors) while using only 33\% more time per training-iteration than the standard BPTT (assuming a single backward step being twice as expensive as a forward step).

\subsection{Backpropagation though time with mixed state memorization (BPTT-MSM)}
\label{sec-mixed}

It is possible to derive an even more general model by combining both approaches as described in Sections \ref{section-hidden} and \ref{section-full}. Suppose we have a total memory capacity $m$ measured in terms of how much a single hidden states can be remembered. Also suppose that storing an internal RNN core state takes $\alpha$ times more memory where $\alpha \geq 2$ is some integer number. We will choose between saving a single hidden state while using a single memory unit and storing an internal RNN core state by using $\alpha$ times more memory. The benefit of storing an internal RNN core state is that we will be able to save a single forward operation during backpropagation.

Define $C(t, m)$ as a computational cost in terms of a total amount of forward operations when running an optimal strategy. We use the following boundary conditions: $C(t, 1) = \frac{1}{2} t(t+1) $ as a cost reflecting the minimal memory approach, while $C(t, m) = t$ for all $m \geq \alpha t$ when memory is plentiful and $C(t - y, m) = \infty$ for all $m \leq 0$ and $C(0, m) = 0$ for notational convenience. We use a similar divide-and-conquer approach to the one used in previous sections.

Define $Q_1(t, m, y)$ as the computational cost if we choose to firstly remember a hidden state at position $y$ and thereafter follow an optimal policy (identical to (~\ref{eq-q})):
\begin{equation}
Q_1(t, m, y) = y + C(y, m) + C(t - y, m - 1)
\end{equation}
Similarly, define $Q_2(t, m, y)$ as the computational cost if we choose to firstly remember an internal state at position $y$ and thereafter follow an optimal policy (similar to (~\ref{eq-q-internal})
 except that now the internal state takes $\alpha$ memory units):
\begin{equation}
\label{eq-q2}
Q_2(t, m, y) =y + C(y - 1, m) + C(t - y, m - \alpha)
\end{equation}
Define $D_1$ as an optimal position of the next push assuming that the next state to be pushed is a hidden state and define $D_2$ as an optimal position if the next push is an internal core state. Note that $D_2$ has a different range over which it is minimized, for the same reasons as in equation \ref{eq-d-internal}:
\begin{equation}
\label{eq-d2}
D_1(t, m) = \argmin_{1 \le y < t} Q_1(t, m, y) \qquad D_2(t, m) = \argmin_{1 \le y \le t} Q_2(t, m, y)
\end{equation}
Also define $C_i(t, m) = Q_i(t, m, D(t, m))$ and finally:
\begin{equation}
\label{eq-c2}
C(t, m) = \min_{i}C_i(t, m) \qquad H(t, m) = \argmin_i C_i(t, m)
\end{equation}
We can solve the above equations by using simple dynamic programming. $H(t, m)$ will indicate whether the next state to be pushed into memory in a hidden state or an internal state, while the respective values if $D_1(t, m)$ and $D_2(t, m)$ will indicate the position of the next push.

\subsection{Removing double hidden-state memorization}
\label{sec-mixed-optimized-2}

Definition \ref{def-internal} of internal RNN core state would typically require for a hidden input state to be included for each memorization.
This may lead to the duplication of information. For example, when an optimal strategy is to remember a few internal RNN core states in sequence, a memorized hidden output of one would be equal to a memorized hidden input for the other one (see Definition \ref{def-internal}).

Every time we want to push an internal RNN core state onto the stack and a previous internal state is already there, we may omit pushing the input hidden state. Recall that an internal core RNN state when an input hidden state is otherwise not known is $\alpha$ times larger than a hidden state. Define $\beta \leq \alpha$ as the space required to memorize the internal core state when an input hidden state is known. A relationship between $\alpha$ and $\beta$ is application-specific, but in many circumstances $\alpha = \beta + 1$. We only have to modify (\ref{eq-q2}) to reflect this optimization:
\begin{equation}
Q_2(t, m, y) =y + C(y - 1, m) + C(t - y, m -  \mathbbm{1}_{y > 1} \alpha  - \mathbbm{1}_{y = 1} \beta)
\end{equation}
$\mathbbm{1}$ is an indicator function. Equations for $H(t, m)$, $D_i(t, m)$ and $C(t, m)$ are identical to (\ref{eq-d2}) and (\ref{eq-c2}).



\subsection{Analytical upper bound for BPTT-HSM}
\label{section-analytic-hm}

We have established a theoretical upper bound for BPTT-HSM algorithm as $C(t, m) \le m t^{1+\frac{1}{m}}$. As the bound is not tight for short sequences, it was also numerically verified that $C(t, m) < 4 t^{1+\frac{1}{m}}$ for $t < 10^5$ and $m < 10^3$, or less than $3 t^{1+\frac{1}{m}}$ if the initial forward pass is excluded. In addition to that, we have established a different bound in the regime where $t < \frac{m^m}{m!}$. For any integer value $a$ and for all $t < \frac{m^a}{a!}$ the computational cost is bounded by $C(t, m) \le (a+1)t$. The proofs are given in the supplementary material. Please refer to supplementary material for discussion on the upper bounds for BPTT-MSM and BPTT-ISM.

\subsection{Comparison of the three different strategies}

\begin{figure}[!ht]
  \centering
\subfigure[Using $10 \alpha$ memory]{\label{mebptt-comparison-10}
\includegraphics[width=0.4\textwidth]{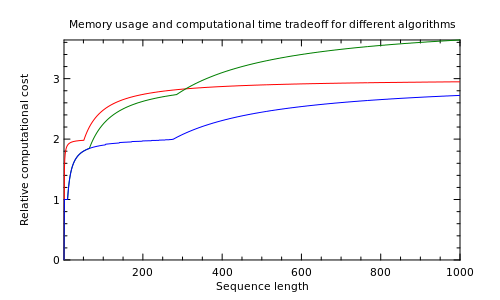}}
\subfigure[Using $20 \alpha$ memory]{  \label{mebptt-comparison-20}
\includegraphics[width=0.4\textwidth]{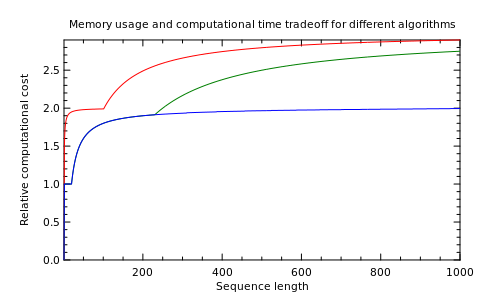}}
  \caption{Comparison of three strategies in the case when a size of an internal RNN core state is $\alpha=5$ times larger than that of the hidden state, and the total memory capacity allows us remember either 10 internal RNN states, or 50 hidden states or any arbitrary mixture of those in the left plot and (20, 100) respectively in the right plot. The red curve illustrates BPTT-HSM, the green curve - BPTT-ISM and the blue curve  - BPTT-MSM. Please note that for large sequence lengths the red curve out-performs the green one, and the blue curve outperforms the other two.}
  \label{mebptt-comparison}
\end{figure}

Computational costs for each previously described strategy and the results are shown in Figure~\ref{mebptt-comparison}. BPTT-MSM outperforms both BPTT-ISM and BPTT-HSM. This is unsurprising, because the search space in that case is a superset of both strategy spaces, while the algorothm finds an optimal strategy within that space. Also, for a fixed memory capacity, the strategy memorizing only hidden states outperforms a strategy memorizing internal RNN core states for long sequences, while the latter outperforms the former for relatively short sequences.

\section{Discussion}

We used an LSTM mapping 256 inputs to 256 with a batch size of 64 and measured execution time for a single gradient descent step (forward and backward operation combined) as a function of sequence length (Figure~\ref{real-times}). Please note that measured computational time also includes the time taken by backward operations at each time-step which dynamic programming equations did not take into the account. A single backward operation is usually twice as expensive than a orward operation, because it involves evaluating gradients both with respect to input data and internal parameters. Still, as the number of backward operations is constant it has no impact on the optimal strategy.

\subsection{Optimality}

The dynamic program finds the optimal computational strategy by construction, subject to memory constraints and a fairly general model that we impose. As both strategies proposed by \citet{chen2016} are consistent with all the assumptions that we have made in section \ref{sec-mixed-optimized-2} when applied to RNNs, BPTT-MSM is guaranteed to perform at least as well under any memory budget and any sequence length. This is because strategies proposed by \citet{chen2016} can be expressed by providing a (potentially suboptimal) policy $D_i(t, m), H(t, m)$ subject to the same equations  for $Q_i(t, m)$.

\subsection{Numerical comparison with Chen's $\sqrt{t}$ algorithm}

Chen's $\sqrt{t}$ algorithm requires to remember $\sqrt{t}$ hidden states and $\sqrt{t}$ internal RNN states (excluding input hidden states), while the recursive approach requires to remember at least $\log_2{t}$ hidden states. In other words, the model does not allow for a fine-grained control over memory usage and rather saves some memory. In the meantime our proposed BPTT-MSM can fit within almost arbitrary constant memory constraints, and this is the main advantage of our algorithm.

\begin{figure}[!ht]
  \centering
    \includegraphics[width=0.8\textwidth]{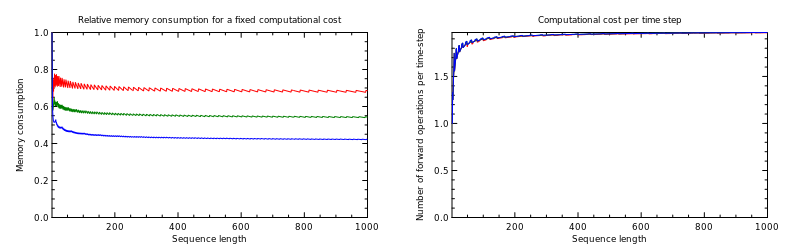}
  \caption{Left: memory consumption divided by $\sqrt{t} (1 + \beta)$ for a fixed computational cost $C=2$. Right: computational cost per time-step for a fixed memory consumption of $\sqrt{t} (1 + \beta)$. Red, green and blue curves correspond to  $\beta=2,5,10$ respectively.}
  \label{comparison-computation}
\end{figure}

The non-recursive Chen's $\sqrt{t}$ approach does not allow to match any particular memory budget making a like-for-like comparison difficult. Instead of fixing the memory budge, it is possible to fix computational cost at $2$ forwards iterations on average to match the cost of the $\sqrt{t}$ algorithm and observe how much memory would our approach use. Memory usage by the $\sqrt{t}$ algorithm would be equivalent to saving $\sqrt{t}$ hidden states and $\sqrt{t}$ internal core states. Lets suppose that the internal RNN core state is $\alpha$ times larger than hidden states. In this case the size of the internal RNN core state excluding the input hidden state is $\beta = \alpha - 1$. This would give a memory usage of Chen's algorithm as $\sqrt{t} (1 + \beta) = \sqrt{t} (\alpha)$, as it needs to remember $\sqrt{t}$ hidden states and $\sqrt{t}$ internal states where input hidden states can be omitted to avoid duplication. Figure \ref{comparison-computation} illustrates memory usage by our algorithm divided by $\sqrt{t} (1 + \beta)$ for a fixed execution speed of $2$ as a function of sequence length and for different values of parameter $\beta$. Values lower than 1 indicate memory savings. As it is seen, we can save a significant amount of memory for the same computational cost.

Another experiment is to measure computational cost for a fixed memory consumption of $\sqrt{t} (1 + \beta)$. The results are shown in Figure \ref{comparison-computation}. Computational cost of $2$ corresponds to Chen's $\sqrt{t}$ algorithm. This illustrates that our approach does not perform significantly faster (although it does not do any worse). This is because Chen's $\sqrt{t}$ strategy is actually near optimal for this particular memory budget. Still, as seen from the previous paragraph, this memory budget is already in the regime of diminishing returns and further memory reductions are possible for almost the same computational cost.

\section{Conclusion}

In this paper, we proposed a novel approach for finding optimal backpropagation strategies for recurrent neural networks for a fixed user-defined memory budget. We have demonstrated that the most general of the algorithms is at least as good as many other used common heuristics. The main advantage of our approach is the ability to tightly fit to almost any user-specified memory constraints gaining maximal computational performance.

\bibliography{bibliography}

\begin{thebibliography}{14}
\providecommand{\natexlab}[1]{#1}
\providecommand{\url}[1]{\texttt{#1}}
\expandafter\ifx\csname urlstyle\endcsname\relax
  \providecommand{\doi}[1]{doi: #1}\else
  \providecommand{\doi}{doi: \begingroup \urlstyle{rm}\Url}\fi

\bibitem[Chen et~al.(2016)Chen, Xu, Zhang, and Guestrin]{chen2016}
Tianqi Chen, Bing Xu, Zhiyuan Zhang, and Carlos Guestrin.
\newblock Training deep nets with sublinear memory cost.
\newblock \emph{arXiv preprint arXiv:1604.06174}, 2016.

\bibitem[Dauvergne and Hasco{\"e}t(2006)]{dauvergne2006data}
Benjamin Dauvergne and Laurent Hasco{\"e}t.
\newblock The data-flow equations of checkpointing in reverse automatic
  differentiation.
\newblock In \emph{Computational Science--ICCS 2006}, pages 566--573. Springer,
  2006.

\bibitem[Eck and Schmidhuber(2002)]{eck2002first}
Douglas Eck and Juergen Schmidhuber.
\newblock A first look at music composition using {LSTM} recurrent neural
  networks.
\newblock \emph{Istituto Dalle Molle Di Studi Sull Intelligenza Artificiale},
  103, 2002.

\bibitem[Graves et~al.(2013)Graves, Mohamed, and Hinton]{graves2013speech}
Alan Graves, Abdel-rahman Mohamed, and Geoffrey Hinton.
\newblock Speech recognition with deep recurrent neural networks.
\newblock In \emph{Acoustics, Speech and Signal Processing (ICASSP), 2013 IEEE
  International Conference on}, pages 6645--6649. IEEE, 2013.

\bibitem[Graves(2012)]{Graves2012}
Alex Graves.
\newblock \emph{Supervised Sequence Labelling with Recurrent Neural Networks}.
\newblock Studies in Computational Intelligence. Springer, 2012.

\bibitem[Graves et~al.(2014)Graves, Wayne, and Danihelka]{graves2014neural}
Alex Graves, Greg Wayne, and Ivo Danihelka.
\newblock Neural turing machines.
\newblock \emph{arXiv preprint arXiv:1410.5401}, 2014.

\bibitem[Grefenstette et~al.(2015)Grefenstette, Hermann, Suleyman, and
  Blunsom]{grefenstette2015learning}
Edward Grefenstette, Karl~Moritz Hermann, Mustafa Suleyman, and Phil Blunsom.
\newblock Learning to transduce with unbounded memory.
\newblock In \emph{Advances in Neural Information Processing Systems}, pages
  1819--1827, 2015.

\bibitem[Gregor et~al.(2015)Gregor, Danihelka, Graves, and
  Wierstra]{gregor2015draw}
Karol Gregor, Ivo Danihelka, Alex Graves, and Daan Wierstra.
\newblock Draw: A recurrent neural network for image generation.
\newblock \emph{arXiv preprint arXiv:1502.04623}, 2015.

\bibitem[Hochreiter and Schmidhuber(1997)]{hochreiter1997long}
Sepp Hochreiter and J{\"u}rgen Schmidhuber.
\newblock Long short-term memory.
\newblock \emph{Neural computation}, 9\penalty0 (8):\penalty0 1735--1780, 1997.

\bibitem[Mnih et~al.(2016)Mnih, Badia, Mirza, Graves, Lillicrap, Harley,
  Silver, and Kavukcuoglu]{mnih2016asynchronous}
Volodymyr Mnih, Adria~Puigdomenech Badia, Mehdi Mirza, Alex Graves, Timothy~P
  Lillicrap, Tim Harley, David Silver, and Koray Kavukcuoglu.
\newblock Asynchronous methods for deep reinforcement learning.
\newblock In \emph{Proceedings of the International Conference on Machine
  Learning (ICML)}, 2016.
\newblock To appear.

\bibitem[Rumelhart et~al.(1985)Rumelhart, Hinton, and
  Williams]{rumelhart1985learning}
David~E Rumelhart, Geoffrey~E Hinton, and Ronald~J Williams.
\newblock Learning internal representations by error propagation.
\newblock Technical report, DTIC Document, 1985.

\bibitem[Sorokin et~al.(2015)Sorokin, Seleznev, Pavlov, Fedorov, and
  Ignateva]{sorokin2015deep}
Ivan Sorokin, Alexey Seleznev, Mikhail Pavlov, Aleksandr Fedorov, and
  Anastasiia Ignateva.
\newblock Deep attention recurrent {Q}-network.
\newblock \emph{arXiv preprint arXiv:1512.01693}, 2015.

\bibitem[Sutskever et~al.(2011)Sutskever, Martens, and
  Hinton]{sutskever2011generating}
Ilya Sutskever, James Martens, and Geoffrey~E Hinton.
\newblock Generating text with recurrent neural networks.
\newblock In \emph{Proceedings of the 28th International Conference on Machine
  Learning (ICML-11)}, pages 1017--1024, 2011.

\bibitem[Werbos(1990)]{werbos1990backpropagation}
Paul~J Werbos.
\newblock Backpropagation through time: what it does and how to do it.
\newblock \emph{Proceedings of the IEEE}, 78\penalty0 (10):\penalty0
  1550--1560, 1990.

\end{thebibliography}

\bibliographystyle{plainnat}

\newpage
\appendix

\section{Pseudocode for BPTT-HSM}

Below is a pseudocode of an algorithm which evaluates an optimal policy for BPTT-HSM. This algorithm has a complexity of $O(t^2 \cdot m)$ but it is possible to optimize to to $O(t \cdot m)$ by exploiting convexity in $t$. In any case, this is a one-off computation which does not have to be repeated while training an RNN.

\newcommand{\tmax}{t_{\small max}}
\newcommand{\mmax}{m_{\small max}}
\newcommand{\Cmin}{C_{\small min}}

\begin{algorithm2e}[h]
\SetKwInOut{Input}{input}\SetKwInOut{Output}{output}
\Input{$\tmax$ -- maximum sequence length; $\mmax$ -- maximum memory capacity}
\textsc{EvaluateStrategy}$(\tmax, \mmax)$ \;
\pushline
Let $C$ and $D$ each be a 2D array of size $\tmax \times \mmax$ \;
\For{$r \in \{ 1, \ldots, \tmax \}$}{
  $C[t][1] \leftarrow \frac{t(t+1)}{2}$ \;
  \For{$m \in \{ t, \ldots, \mmax \}$}{
    $C[t][m] \leftarrow 2t-1$ \;
    $D[t][m] \leftarrow 1$ \;
  }
}
\For{$m \in \{ 2, \ldots, \mmax \}$}{
  \For{$t \in \{ m+1, \ldots, \tmax \}$}{
    $\Cmin \leftarrow \infty$ \;
    \For{$y \in \{ 1, \ldots, t-1 \}$}{
      $c \leftarrow y + C[y][m] + C[t-y][m-1]$ \;
      \If{$c < \Cmin$}{
        $\Cmin \leftarrow c$ \;
        $D[t][m] \leftarrow y$ \;
      }
    }
    $C[t][m] \leftarrow \Cmin$
  }
}
\Return $(C,D)$ \;
\popline
\caption{BPTT-HSM strategy evaluation.}\label{alg:bptt-hsm-eval}
\end{algorithm2e}

Algorithm~\ref{alg:bptt-hsm-exec}, shown below, contains pseudocode
which executes the precomputed policy.

\begin{algorithm2e}[h!]
\SetKwInOut{Input}{input}\SetKwInOut{Output}{output}
\Input{$D$ -- previously evaluated policy;
rnnCore -- mutable RNN core network;
stack -- a stack containing memorized hidden states;
gradHidden -- a gradient with respect to the last hidden state;
$m$ -- memory capacity available in the stack;
$t$ -- subsequence length;
$s$ -- starting subsequence index;}
\textsc{ExecuteStrategy}($D$, rnnCore, stack, gradHidden, $m$, $t$, $s$) \;
\pushline
  hiddenState = \textsc{Peek}(stack) \;
  \If{$t = 0$}{
    \Return gradHidden \;
  }
  \ElseIf{$t = 1$}{
    output $\leftarrow$ \textsc{Forward}(rnnCore, \textsc{GetInput}($s$), hiddenState) \;
    gradOutput $\leftarrow$ \textsc{SetOutputAndGetGradOutput}($s + t$, output) \;
    (gradInput, gradHiddenPrevious) $\leftarrow$ \textsc{Backward}(rnnCore, \textsc{GetInput}($s$), hiddenState, \;
    \hspace{7cm}gradOuput, gradHidden) \;
    \textsc{SetGradInput}($s + t$, gradInput) \;
    \Return gradHiddenPrevious
  }
  \Else{
    $y \leftarrow D[t][m]$ \;
    \For{$i \in \{ 0, \ldots, t-1 \}$}{
      output, hiddenState $\leftarrow$ \textsc{Forward}(rnnCore, \textsc{GetInput}($s + i$), hiddenState) \;
    }
    \textsc{Push}(stack, hiddenState) \;
    gradHiddenR $\leftarrow$ \textsc{ExecuteStrategy}($D$, rnnCore, stack, gradHidden, $m-1$, $t-y$, $s+y$) \;
    \textsc{Pop}(stack) \;
    gradHiddenL $\leftarrow$ \textsc{ExecuteStrategy}($D$, rnnCore, stack, gradHiddenR, $m$, $y$, $s$) \;
    \Return gradHiddenL
  }
\popline
\caption{BPTT-HSM strategy execution.}\label{alg:bptt-hsm-exec}
\end{algorithm2e}

\section{Upper bound of the computational costs for BPTT-SHM}

\subsection{General upper bound}

Consider the following dynamic program
\begin{equation}
C(t, m) = \min_{1 \le y < t} (y + C(t - y, m - 1) + C(y, m))
\end{equation}
with boundary conditions: $C(t, 1) = \frac{1}{2} t(t+2) $ and $C(t, m) = 2t - 1$ for all $m \geq t$

\begin{proposition}
 We have $C(t,m) \leq  m t^{1+1/m}$ for all $t,m \geq 1$.
\end{proposition}

\begin{proof}
It is straightforward to check that the bound is satisfied at the boundaries. Now let us define the boolean functions $A(t, m) := \{ C(t,m) \leq m t^{1+1/m} \}$ and $A(m) := \{ \forall t\geq 1, C(t,m) \leq  m t^{1+1/m} \}$.

Let us prove by induction on $m$ that $A(m)$ is true. Thus assume $A(m)$ is true and let us prove that $A(m+1)$ is also true. For that, we will prove by induction on $t$ that $A(t,m+1)$ is true. Thus for any $t\geq 2$, assume $A(y,m+1)$ is true for all $y<t$ and let us prove that $A(t,m+1)$ is also true.

We have
\beqa
C(t,m+1) &=&  \min_{1 \leq y < t} \big[ y + C(t - y, m) + C(y, m+1)\big] \\
&=&  \min_{1 \leq y < t} \big[ y + m (t-y)^{1+1/m} + (m+1) y^{1+1/(m+1)} \big]  \label{eq:C.bound.1}
\eeqa
using our inductive assumption that $A(y,m+1)$ and $A(t-y, m)$ are true.

For any real number $y\in[1,t-1]$, define $g(y) = y + m (t-y)^{1+1/m} + (m+1) y^{1+1/(m+1)}$. $g$ is convex (as the sum of 3 convex functions) and is smooth over $[1,t-1]$ with
$$g'(y) = 1 - (m+1) (t-y)^{1/m} + (m+2) y^{1/(m+1)},$$
and
$$g''(y) = \frac{m+1}{m} (t-y)^{1/m-1} + \frac{m+2}{m+1} y^{1/(m+1)-1}.$$
Notice that $g''$ is positive and convex, thus 
\beq 
\max_{1\leq s\leq t-1} |g''(s)| = \max ( g''(1),g''(t-1) ) \leq \frac{m+1}{m} (1+(t-1)^{1/m-1})\leq 4.\label{eq:max.g''}
\eeq

Let $y^*$ the (unique) optimum of $g$ (i.e., such that $g'(y^*)=0$). Then we have
\beqan
C(t,m+1) & \leq & g(\lfloor y^* \rfloor) \\
& \leq & g(y^*) + ( y^* - \lfloor y^* \rfloor ) g'(y^*) + \frac{1}{2} ( y^* - \lfloor y^* \rfloor )^2 \max_{1\leq s\leq t-1} |g''(s)|\\
& \leq & g(y^*) +  2 \\
& \leq & g(\tilde y) +  2
\eeqan
where $\tilde y$ defined as the unique solution to 
\beq
t- y = y^{m/(m+1)}.\label{eq:y}
\eeq

Notice that for any $t\geq 2$, $\tilde y\in [1,t)$. We deduce from \eqref{eq:C.bound.1} that
\beqa
C(t,m+1) &\leq& \tilde y + m (t-\tilde y)^{1+1/m} + (m+1) \tilde y^{1+1/(m+1)} + 2 \notag \\
&\leq&  \tilde y + m \tilde y + (m+1) \tilde y^{1+1/(m+1)} +  2 \label{C.bound.2} 
\eeqa

Now, using the convexity of $x\mapsto x^{1+1/(m+1)}$ and since $y<t$, we have
\beqa
t^{1+1/(m+1)} &\geq& \tilde y^{1+1/(m+1)} + (t-\tilde y) (1+\frac{1}{m+1}) \tilde y^{1/(m+1)} \notag \\
& =  & \tilde y^{1+1/(m+1)} + \tilde y^{m/(m+1)} (1+\frac{1}{m+1}) \tilde  y \notag \\
& =  & \tilde y^{1+1/(m+1)} + (1+\frac{1}{m+1}) \tilde y \label{eq:t.bound}
\eeqa
(where the last equality derives from the definition of $\tilde y$ in \eqref{eq:y}). Now putting \eqref{eq:t.bound} into \eqref{C.bound.2} we deduce:
\beqan
C(t,m+1) &\leq& \tilde y + m \tilde y  + (m+1) \big[ t^{1+1/(m+1)} - (1+\frac{1}{m+1}) \tilde y \big] + 2 \\
&=& (m+1) t^{1+1/(m+1)} - \tilde y + 2 \\
&\leq&  (m+1) t^{1+1/(m+1)},
\eeqan
as soon as $ \tilde y \geq 2$, which happens when $t\geq 4$. Now the cases $t <4$ (which actually corresponds to the 2 cases: ($t=3,m=2$)  and ($t=3, m=3$)) are verified numerically.
\end{proof}

\subsection{Upper bound for short sequences}

The algorithm described in the previous section finds an optimal solution using dynamic programming. It is trival to show that $C(t, m)$ is an increasing function in $t$ and we will make use of this property. It is possible to prove a computational upper bound by finding a potentially sub-optimal policy and evaluating its cost. Alternatively, one can find a policy for a given computational cost, and then use this as an upper bound for an optimal policy.

It was established that when memory equals to the sequence ($t = m$) length then $C(t, m) = 2t - 1 < 2t$. Define $T(a, m)$ as the maximum sequence length $t$ for which an average computational cost $C(t, m)/t \le a$. Hence, we can clearly see that $T(2, m) \ge m$.

\begin{proposition}
\label{prop-T}
 We have $T(a, m) \ge \frac{m^{a-1}}{(a-1)!}$
\end{proposition}

\begin{proof}

This is clearly satisfied for the case $a=2$. Assume that proposition is true for some value $a$. We prove by induction that this is also satisfied for all other values of $a$. 

\begin{figure}[!ht]
  \centering
    \includegraphics[width=1.0\textwidth]{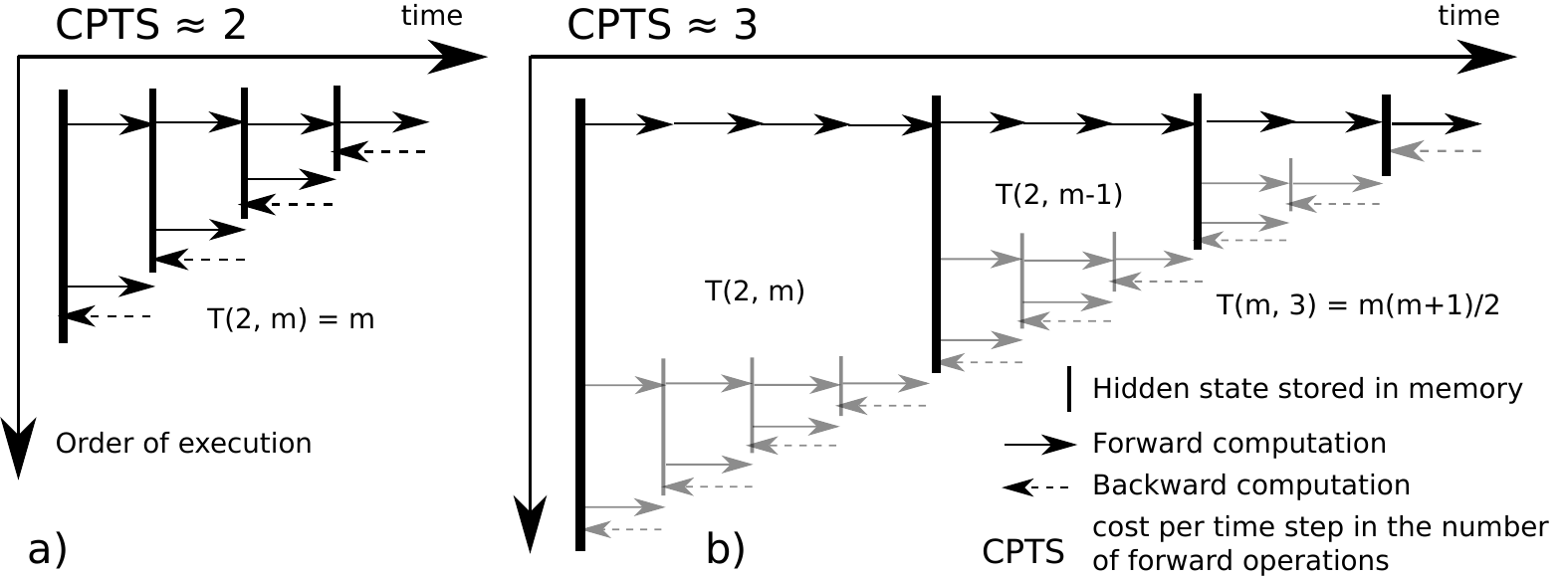}
  \caption{Illustration of the optimal strategy for two easy cases: a) $t = 4, m = 4$ and b) $t = 10, m = 4$. $T(a, m)$ is the maximum sequence length for which $C(t, m) \le at$, where $a=$ CPTS (cost per time step). Note that we can construct more expensive fixed-CPTS strategies by composing cheaper fixed-CPTS strategies. Each sub-strategy in Figure b has one less memory slot available when going to the right. In the case of arbitrary sequence lengths strategies may not look as nice.}
  \label{drawing-heurictics}
\end{figure}

Consider a sequence of length $t = \sum_{i=1}^{m} T(a, i)$. We will make the first initial pass over the sequence at the cost $t$ and will memorize hidden states spaced at intervals $T(a, m)$, $T(a, m-1)$ .. $T(a, 2)$ (Figure \ref{drawing-heurictics}). Once the hidden states are memorized, we can run the same backpropagation algorithm on each sub-sequence each time paying the cost of $\le a$ per time-step (Proposition \ref{prop-T}). This will make the local cost of the algorithm $C(t, a+1) \le t + ta = (a+1)t$. As this makes the cost per time-step $\le a$, $T(a+1, m) \ge t = \sum_{i=1}^{m} T(a, i) \ge \sum_{i=1}^{m} \frac{i^{a-1}}{(a-1)!} \ge \frac{m^{a}}{{a!}}$.

\end{proof}

This implies that $C(t, m) \le (a+1)t$ when $t \le \frac{m^{a}}{a!}$. Please note that this is a much better result comparing to a strategy when interval is sub-divided into equal-sub-intervals several times recursively, as the later strategy would only give $C(t, m) \le (a+1)t$ for $t \le \frac{m^{a}}{a^a}$ while ($a! \ll a^a$) for the same computational cost. It turns out that it is non-trivial to invert this function to state computation explicitly as a function of $t, m$.

\subsection{Analytical upper bounds for BPTT-ISM and BPTT-MSM}

When we are capable of memorizing internal core states instead of hidden states, we can apply almost exactly the same calculation of the upper bound as in Section \ref{section-analytic-hm} while still being conservative. The main difference though a removal of a single forward operation per time-step. This would give us the upper bound of the computational cost as $C(t, m) \le at$ for $t \le \frac{m^{a}}{a!}$. The same upper bound of $C(t, m) \le m t^{1+\frac{1}{m}}$ can be assumed to be true for the case when  $t > \frac{m^{a}}{a!}$, because an internal state also includes a hidden state, and the derived optimal algorithm can not do any worse for the same number of $m$. Please note that in this case the same number of $m$ constitutes for a larger actual memory usage, as the units of measurement are different. $T(a, m)$ is the maximum sequence length for which a computational cost $C(T(a, m), m) \le at$.

Any upper bounds derived for the case of BPTT-HSM will also hold for the case of BPTT-MSM, because the later is generalization of the former technique, and an optimal policy found will be at least as good as.

\section{Generalizing to deep feed-forward networks}

BPTT-MSM can also be generalized to deep network architectures rather than RNNs as long as the computational graph is a linear chain. The main difference is that different layers in deep architectures are non-homogeneous (\ie have different computational costs and memory requirements) while cores within the same RNN are homogeneous. This difference can be addressed by modifying the dynamic program formulation.

To start with, lets derive a strategy when only hidden states are stored. We can recall that an optimal policy of BPTT-HSM algorithm could be found by solving given dynamic programming equations.

Cost if we choose to memorize next state at position $y$ and thereafter we follow an optimal policy:
\begin{equation}
\label{eq-q}
Q(t, m, y) = y + C(t - y, m - 1) + C(y, m)
\end{equation}

Optimal position of the next memorization:
\begin{equation}
\label{eq-d}
D(t, m) =  \argmin_{1 \le y < t} Q(t, m, y)
\end{equation}

Computational cost under the optimal policy:
\begin{equation}
\label{eq-c}
C(t, m) = Q(t, m, D(t, m))
\end{equation}

As in the case of RNNs, we choose to remember only some of the intermediate output results (hidden states) and recompute all internal states (and other hidden states) on demand while fitting within some memory allowance. As sizes of internal representations of different layers are different, it is necessary to include a size of a "working" network layer into our current memory allowance. In the case of RNNs this constant factor could be comfortably ignored. In addition to that, hidden states produced by different layers will also have different sizes.

Suppose that the cost of recomputing layer $y$ is $u_y$ while the size of a hidden state computed after step $y$ has the size of $s_y$. We also assume that the initial input vector is always available. 

Define $U(x, y) = \sum^{x+y}_{i = x+1} u_i$ as a cumulative computational cost of forward propagation when $x$ bottom layers are cut-off and we are forward propagating over $y$ layers . Also, define $p_i$ to be the size of an internal state of some given network layer: this defined as the minimum memory requirement to execute forward and backward operations on a given layer. Neither forward not backward operation is impossible if we have less memory left that the operation requires. Define the maximum memory usage when executing forward operation on layers from $x+1$ to $y$ inclusive as $P(x, y) = \max_{x < i \le y} p_i$. For the reasons discussed previously it is convenient to set computational cost to infinity when we have have less than required memory available: $K(x, y, m) = 0$ if $P(x, y) \le m$ and $K(x, y, m) = \infty$ if $P(x, y) > m$.

Consider a part of the neural network with $x$ bottom layers cut-off. Define $C(t, m - 1, x)$ as a computation cost of a combined forward and back-propagation on such network over $t$ bottom layers.

The cost of a combined forward  and back-propagation of the cut-off section assuming that the next memorization happens at position $y$ is:
\begin{equation}
Q(t, m, y, x) = U(x, y) + K(x, y, m) + C(t - y, m - s_{x+y}, x + y) + C(y, m, x)
\end{equation}

$K(x, y, m)$ prevents us from making an impossible back-propagation commitment when memory is not enough.

It is now trivial to define position of the next memorization as:
\begin{equation}
D(t, m, x) =  \argmin_{1 \le y < t} Q(t, m, y, x)
\end{equation}

And hence we can evaluate a cost under the optimal policy:
\begin{equation}
C(t, m, x) = Q(t, m, D(t, m, x), x)
\end{equation}

As previously, we have to set boundary conditions:

When no extra memory left:
\begin{equation}
C(t, 0, x) = \sum^{x+t}_{i = x} (t - i + 1)u_i
\end{equation}

When the memory is plentiful ($m >= \sum^{x+t}_{i=x} s_i $):
\begin{equation}
C(t, m, x) = u_t + \sum^{x+t-1}_{i = x} 2 u_i
\end{equation}

Another boundary condition applies when the network has zero depth: $C(t, m, x) = 0$ when $x + t > N$ and $N$ is the number of layer in the network and also $C(t, m, x) = 0$ when $t \le 0$. It is convenient to set $C(t, m, x) = \infty$ for $m < 0$ to emphasize that memory can never become negative.

We can solve the equations using dynamic programming, but unlike in the recurrent case when dynamic programming required filling a 3D rather than a 2D array. This means that evaluation of the strategy may become impractical for sequences longer than a few hundred, but a good thing is that a strategy is a one-off computational.

The algorithm is executed as follows: if any any point we start a recursive call of the algorithm at layer $x$ while having memory allowance $m$, we evaluate $y = D(t, m, x)$, forward-propagate states until $y$, memorize the next state and call the algorithm recursively both parts of the sequence.

Similarly equations can be generalized for BPTT-ISM and BPTT-MSM algorithms,

\end{document}